\documentclass[journal]{article} 

\usepackage{times} 
\usepackage{helvet} 
\usepackage{courier}  
\usepackage[hyphens]{url}  
\usepackage{graphicx} 
\usepackage[margin=1in]{geometry}
\usepackage{cite}  
\usepackage{caption} 
\DeclareCaptionStyle{ruled}{labelfont=normalfont,labelsep=colon,strut=off} 
\usepackage{algorithm}
\usepackage{newfloat}
\usepackage{listings}
\floatstyle{ruled}
\newfloat{listing}{tb}{lst}{}
\floatname{listing}{Listing}

\usepackage{url,multirow,comment}
\usepackage{amsmath,amssymb,amsthm}
\usepackage{subcaption}
\usepackage[compatible]{algpseudocode}
\usepackage[flushleft]{threeparttable}
\usepackage{graphicx,color}
\usepackage[figuresright]{rotating}
\usepackage[utf8]{inputenc} 
\usepackage[T1]{fontenc}    
\usepackage[dvipsnames]{xcolor}

\usepackage{array,booktabs,ragged2e}    
\usepackage{enumitem}
\usepackage{amsfonts}       
\usepackage{amssymb,mathtools}
\usepackage{nicefrac}       
\usepackage{microtype}      %
\usepackage{color}
\usepackage{pbox}

\usepackage{cite}
\usepackage{thmtools}
\usepackage{thm-restate}
\usepackage{enumitem}
\usepackage{bbm}
\usepackage[bookmarks=false,colorlinks=true,urlcolor=blue,citecolor=blue,linkcolor=blue]{hyperref}

\usepackage{bm}

\newcolumntype{R}[1]{>{\RaggedRight\arraybackslash}p{#1}}

\newtheorem{lem}{Lemma}

\newtheorem{thm}{Theorem}

\newtheorem{defn}{Definition}

\newtheorem{example}{Example}

\newtheorem{cexample}{Canonical Example}

\newtheorem{propty}{Property}

\newtheorem{rem}{Remark}

\newtheorem{candmeas}{Candidate Measure}

\newcommand{\iid}[0]{i.i.d.}

\newcommand{\uni}[2]{\mathrm{Uni}({#1: #2})}
\newcommand{\red}[2]{\mathrm{Red}({#1: #2})}
\newcommand{\syn}[2]{\mathrm{Syn}({#1: #2})}
\newcommand{\mut}[2]{\mathrm{I}({#1; #2})}

\newcommand{\given}[0]{\mid}

\makeatletter
\def\thanks#1{\protected@xdef\@thanks{\@thanks
        \protect\footnotetext{#1}}}
\makeatother
\begin{document}
\date{}
\title{Quantifying Feature Contributions to Overall Disparity Using Information Theory}
\author{
Sanghamitra Dutta, Praveen Venkatesh, Pulkit Grover \\  
\thanks{Presented at the AAAI-22 Workshop on Information-Theoretic Methods for Causal Inference and Discovery in March 2022.}%
\thanks{All the authors were affiliated with the Department of Electrical and Computer Engineering at Carnegie Mellon University. S. Dutta is currently a research scientist at JP Morgan AI Research. P. Venkatesh is currently a scientist at the Allen Institute and a postdoctoral scholar at the University of Washington, Seattle. The authors acknowledge the support of NSF, CMU Cylab, and the Simons Institute. Author Contacts: S. Dutta (sanghamitra2612@gmail.com), P. Venkatesh (praveen.venkatesh@alleninstitute.org), P. Grover (pgrover@andrew.cmu.edu).}
}

\maketitle

\begin{abstract}
When a machine-learning algorithm makes biased decisions, it can be helpful to understand the ``sources'' of disparity to explain why the bias exists. Towards this, we examine the problem of quantifying the contribution of each individual feature to the observed disparity. If we have access to the decision-making model, one potential approach (inspired from intervention-based approaches in explainability literature) is to vary each individual feature (while keeping the others fixed), and use the resulting change in disparity to quantify its contribution. However, we may not have access to the model or be able to test/audit its outputs for individually varying features. Furthermore, the decision may not always be a deterministic function of the input features (e.g., with human-in-the-loop). For these situations, we might need to explain contributions using purely distributional (i.e., observational) techniques, rather than interventional. We ask the question: \emph{what is the ``potential'' contribution of each individual feature to the observed disparity in the decisions when the exact decision-making mechanism is not accessible?} We first provide canonical examples (thought experiments) that help illustrate the difference between distributional and interventional approaches to explaining contributions, and when either is better suited. When unable to intervene on the inputs, we quantify the ``redundant'' statistical dependency about the protected attribute that is present in both the final decision and an individual feature, by leveraging a body of work in information theory called Partial Information Decomposition. We also perform a simple case study to show how this technique could be applied to quantify contributions.
\end{abstract}

\section{INTRODUCTION}
\label{sec:introduction}

Machine learning algorithms have permeated almost every aspect of our lives, including high-stakes applications, e.g. hiring and admissions. With their growing use, it has become increasingly important to incorporate algorithmic fairness to avoid bias with respect to protected attributes, such as, gender, race, nationality, age, etc. Existing literature on fairness~\cite{dwork2012fairness,kamiran2013quantifying,agarwal2018reductions,hardt2016equality, varshney2019trustworthy,calmon2017optimized, dutta2020tradeoff, fairMI,wang2021split,datta2017use,kilbertus2017avoiding,zhang2018fairness,corbett2017,nabi2018fair,chiappa2018path,interventional_fairness,dutta2020information,dutta2021fairness,xu2020algorithmic,chen2019fairness} provides several measures of bias and disparity at the final model output, as well as, several techniques to mitigate these biases during model design.

However, in many applications, e.g., college admissions, the decision-making mechanism is a complex combination of algorithms and human-in-the-loop. Thus, only identifying bias and disparity in the final decision may not be enough to audit and, subsequently, mitigate them. E.g., there is an ongoing debate in the US on whether GRE/TOEFL scores should be used for college admissions because they may cause disparity in the decisions with respect to protected attributes~\cite{gre_ets, gre_atlantic}. It would help to understand \textit{how} the disparity in the decisions arose, e.g., which features could be potentially responsible for the disparity, and then evaluate how critical those features are for the specific application. In fact, several existing discrimination laws (e.g. Title VII of the US Civil Rights Act~\cite{title7, barocas2016big}) allow exemptions if the disparity can be justified by an occupational necessity, e.g., coding test for software engineers, or weightlifting ability for  firefighters~\cite{kamiran2013quantifying,grover1995business,corbett2017,datta2017use,dutta2020information,dutta2021fairness,kilbertus2017avoiding,zhang2018fairness,nabi2018fair,chiappa2018path,interventional_fairness}. 

The problem of quantifying contribution of different features to the overall disparity bridges the fields of both fairness and explainability. Explanations often provide an understanding (e.g. through visualization) of the contribution of each individual feature to the final decision~\cite{datta2016algorithmic,lundberg2017unified,ribeiro2016should}. Here we are interested in quantifying the contribution of features to the observed disparity in the decision making. 

One possible approach inspired from existing techniques in explainability (e.g., QII~\cite{datta2016algorithmic}, Influence Functions~\cite{koh2017understanding}, SHAP~\cite{lundberg2017unified,ghorbani2020neuron}) could be to intervene on the input features and observe changes in the output. For instance, one can vary an input feature while keeping other features unchanged, and assign the resulting change in disparity in the final output to this feature to quantify its contribution to the disparity. However, these techniques are not a good fit when one does not have access to the model to be able to intervene on its inputs.

Furthermore, in many applications, e.g., admissions, the machine-learnt algorithm might be used in conjunction with human-in-the-loop. This introduces an additional challenge in quantifying contributions of individual features: the output is no longer a deterministic function of the input features. The final decision may not be the same for two candidates having the same values of input features as additional, non-quantified aspects are taken into consideration.

In this work, the question we ask is: what is the potential contribution of each feature to the observed disparity when the exact decision-making mechanism is not accessible? To answer this question, we introduce two perspectives to this problem: (i) an interventional perspective; and (ii) a purely distributional perspective (i.e., observational). We then demonstrate when one is better suited than the other. We call our first approach \emph{interventional} because it requires one to intervene on specific input features and observe the final outputs, thereby changing their joint distribution. We call our second approach \emph{distributional/observational} because it only depends on the joint distribution of the observed input features and the final output, and does not require access to the model to be able to intervene on specific input features.



Our work makes the following contributions:

\noindent \textbf{1. Explaining Contribution of Individual Features to the Observed Disparity:}
We first quantify the observed disparity as the mutual information~\cite{cover2012elements} $\mut{Z}{\hat{Y}}$ where $Z$ is the protected attribute, e.g., gender, race, etc., and $\hat{Y}$ is the final decision. Then, we introduce an \emph{interventional} approach (see Candidate Measure~\ref{candmeas:interventional} in Section~\ref{sec:main_results}) to quantify the contribution of each individual feature to the observed disparity. This approach is inspired from existing works~\cite{datta2016algorithmic,lundberg2017unified} in explainability, and can be used when one has access to the model making the decisions, and is able to intervene on the inputs to observe the outputs (decisions). However, for situations when access to the model is either not available or the decisions are not deterministic given the input features, we also introduce a \emph{distributional} approach (see Candidate Measure~\ref{candmeas:distributional} in Section~\ref{sec:main_results}). We provide a measure to quantify the ``potential'' contribution of each individual feature to the decision. This technique leverages a body of work in information theory called Partial Information Decomposition (PID) that quantifies the ``redundant'' information about the protected attribute $Z$ present both in $\hat{Y}$ and an input feature.

\begin{rem}
We use the term ``potential'' because even if a feature has a non-zero potential contribution (by our definition), one is unable to check if intervening on this feature actually results in a change in disparity in the final decision. Being unable to intervene on the model inputs, we essentially capture the ``redundant'' statistical dependency (with the protected attribute) that is present in both an individual feature and the final decision. Thus, in spirit, this may seem similar to capturing statistical correlations with the input features and the output, as discussed in explainability literature~\cite{Molnar2019Interpretable}. However, note that, our quantification involves three random variables: we want to capture the statistical dependency \emph{about the protected attribute} that is present in both \emph{the final decision} and \emph{an input feature}. Statistical correlation is only defined for two random variables. This leads us to examine ``redundant'' information that precisely quantifies this dependency.
\end{rem}

\noindent \textbf{2. Canonical Examples to Illustrate the Difference Between Various Explainability Approaches:}
We also discuss several canonical examples in Section~\ref{sec:main_results} that illustrate the differences between an \emph{interventional} and a \emph{distributional} approach to explaining the contribution of each individual feature. These examples also help us understand when one approach is better suited than the other.

\noindent \textbf{3. Case Study on an artificial admissions dataset:} We finally demonstrate a case study on an artificial graduate admissions dataset to demonstrate how these techniques could be used to explain contributions.

\textbf{Related Works:} Algorithmic fairness is an important field of research with several measures of fairness as well as approaches to incorporate them in model design. The most closely related works to our work are \cite{kamiran2013quantifying,zhang2018fairness,corbett2017,nabi2018fair,chiappa2018path,interventional_fairness,dutta2020information,dutta2021fairness,galhotra2020fair,khodadadian2021information,xu2020algorithmic}. However, \cite{kamiran2013quantifying,zhang2018fairness,corbett2017,nabi2018fair,chiappa2018path,interventional_fairness,dutta2020information,dutta2021fairness,xu2020algorithmic} focus on quantifying exempt and non-exempt discrimination (either using observational measures or causal modelling) given a choice of critical features, rather than quantifying the contribution of each individual feature. Alternatively, \cite{galhotra2020fair,khodadadian2021information} propose information-theoretic techniques to carefully select features for fair decision making. In this work, we focus on explaining the contributions of all the individual features to the overall disparity, even when we may not have access to the exact decision-making mechanism (including scenarios with human-in-the-loop). In \cite{pan2021explaining}, the focus is on leveraging the underlying causal graph to understand feature contributions to disparity. Another recent related work~\cite{ge2022explainable} focus on counterfactual explainability techniques for fairness.

This work is also connected with the literature on explainability~\cite{datta2016algorithmic,lundberg2017unified,Olah2018Building,ghorbani2020neuron,Sundararajan2017Axiomatic,Dabkowski2017Real,Bhatt2020Evaluating,koh2017understanding,Kim2016Examples,harutyunyan2021estimating,venkatesh2020information,student_risk,adler2018auditing,kumar2020problems} (see \cite{Molnar2019Interpretable} for a survey). In particular, \cite{lundberg2017unified,ghorbani2020neuron}  use Shapley values, but their goal is to quantify the contribution of individual features to the \textit{decision} (or, its accuracy), and propose several approximations for the ease of computation. Here, our goal is not to explain the contribution of individual features to the \textit{decision} (or, its accuracy), but rather their contribution to the overall \textit{disparity in the decision}. We note that \cite{ghorbani2020neuron} also extend their Shapley-value-based explainability technique for an application in fairness: quantifying contributions of neurons to the accuracy only on a protected group, e.g., people of a certain gender or race. In contrast, here we are interested in quantifying contribution of different features to the mutual information $\mut{Z}{\hat{Y}}$ (statistical dependence between protected attribute and final output). In this context, our work brings out the contrast between interventional and distributional approaches to explainability through canonical examples, and illustrates when one is better suited than the other. In doing so, our work leverages tools from Partial Information Decomposition, that also can have broad applications in explainability.


\begin{figure}
\centering
\includegraphics[height=3.5cm]{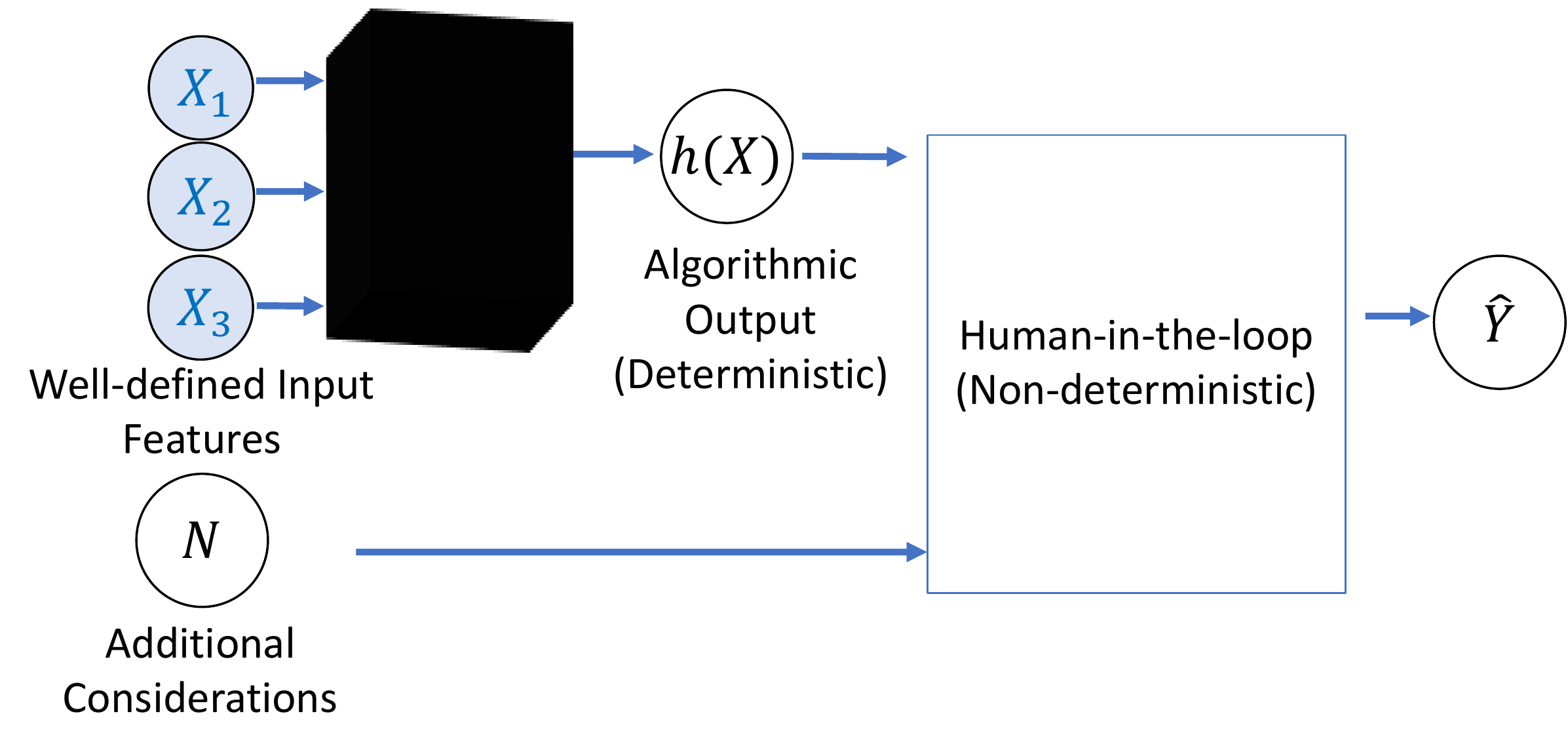}
\caption{\small{Illustration of our setup: Here $X=(X_1,X_2,\ldots,X_n)$ denotes a set of well-defined input features that are input to a (deterministic) model $h(X)$. The final decision $\hat{Y}$ combines $h(X)$ and additional considerations (which usually includes the protected attribute $Z$). Thus, $\hat{Y}$ is not necessarily a deterministic function of $X$. \label{fig:system} }}
\end{figure}

\section{PRELIMINARIES}

\subsection{Our Notations and System Model}
\label{system_model}
We let $X=(X_1,X_2,\ldots,X_n)$ denote a set of well-defined input features that go into a model. The output of the model is denoted by $h(X)$ which is usually a deterministic function of the model inputs. We also let $Z$ denote the protected attribute, e.g., gender, race, etc., which may or may not be an input feature explicitly fed into the model ($h(X)$). The final decision is denoted by $\hat{Y}$ which is a complex combination of the deterministic model output $h(X)$ and subjective evaluation by human-in-the-loop who may take additional factors (non-quantified aspects) into consideration. These additional factors almost always include the protected attribute $Z$, e.g., gender, race, age, etc. Therefore, the final decision $\hat{Y}$ may not be a deterministic function of the model inputs $X$, and could also depend on $Z$ (see Fig.~\ref{fig:system}).

Next, we provide a brief background on Partial Information Decomposition (PID) in Section~\ref{background:pid} and Shapley Values in Section~\ref{background:shapley} for completeness. Readers familiar with these concepts may directly proceed to our main results in Section~\ref{sec:main_results}, which is followed by a case study on an artificial admissions dataset in Section~\ref{sec:case_study}.

\subsection{Background on Partial Information Decomposition (PID)}
\label{background:pid}

PID~\cite{bertschinger2014quantifying,williams2010nonnegative,griffith2014quantifying} is an emerging body of work in information theory that decomposes the mutual information $\mut{Z}{(A,B)}$ about a random variable $Z$ contained in the tuple $(A,B)$ into four \emph{non-negative} terms (also see Fig.~\ref{fig:pid}):
\begin{align}
 & \mut{Z}{(A,B)} = \uni{Z}{A| B} + \uni{Z}{B| A} \nonumber\\
 & + \red{Z}{(A, B)} + \syn{Z}{(A, B)}. \label{eq:pid1}
\end{align}
Here, $\uni{Z}{A| B}$ denotes the unique information about $Z$ that is present only in $A$ and not in $B$. Similarly, $\uni{Z}{B| A}$ is the unique information about $Z$ that is present only in $B$ and not in $A$. The term $\red{Z}{(A, B)}$ denotes  the redundant information about $Z$ that is present in both $A$ and $B$, and $\syn{Z}{(A, B)}$ denotes the synergistic information not present in either of $A$ or $B$ individually, but present jointly in $(A,B)$. \emph{All four of these terms are non-negative. Also notice that, $\red{Z}{(A, B)}$ and $\syn{Z}{(A, B)}$ are symmetric in $A$ and $B$.}  Before defining these PID terms formally, let us understand them through an intuitive example.
\begin{figure}
\centering
\includegraphics[height=3cm]{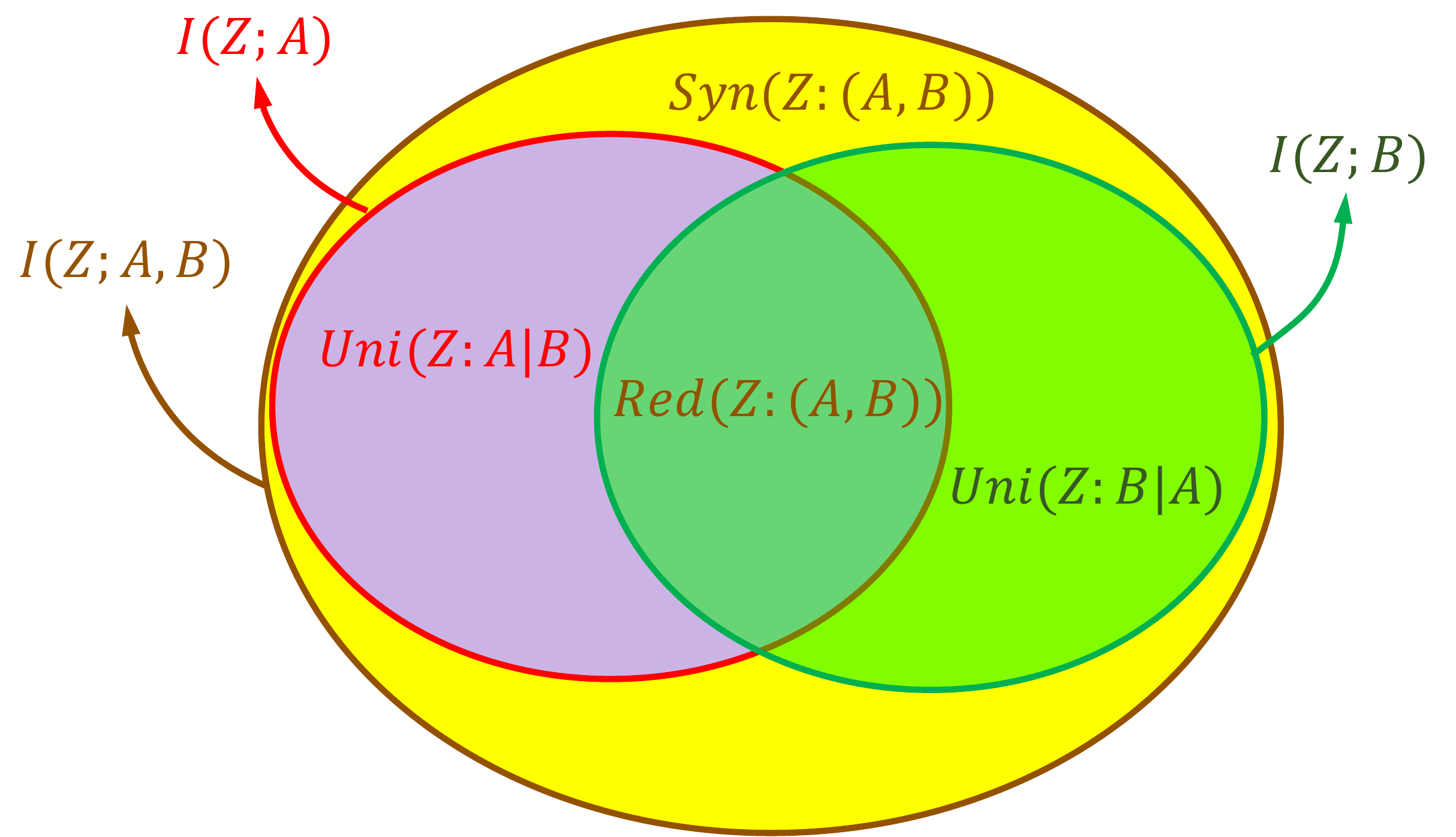}
\caption{\small{PID: Mutual information $\mathrm{I}(Z;(A,B))$ is decomposed into $4$ non-negative terms, namely, $\uni{Z}{A| B}$, $\uni{Z}{B| A}$, $\red{Z}{(A, B)}$ and $\syn{Z} {(A, B)}$. Also note that, $\mathrm{I}(Z;(A,B))=\mathrm{I}(Z;B)+\mathrm{I}(Z;A\given B),$ each of which is in turn a sum of two PID terms. $\red{Z}{(A, B)}$ is the sub-volume between $\mathrm{I}(Z;A)$ and $\mathrm{I}(Z;B)$, and $\uni{Z}{A| B}$ is the sub-volume between $\mathrm{I}(Z;A\given B)$ and $\mathrm{I}(Z;A)$.  \label{fig:pid} }}
\end{figure}
\begin{example}[Understanding PID]
Let $Z=(Z_1,Z_2,Z_3)$ with $Z_1,Z_2,Z_3\sim$ \iid{} Bern(\nicefrac{1}{2}). Let $A=(Z_1,Z_2,Z_3\oplus N)$, $B=(Z_2,N)$, $N\sim $  Bern(\nicefrac{1}{2}) is independent of $Z$. Here,  $\mathrm{I}(Z; (A,B))=3$ bits. 
\end{example}
\noindent The unique information about $Z$ that is contained only in $A$ and not in $B$ is effectively contained in $Z_1$ and is given by $\uni{Z}{A| B} = \mut{Z}{Z_1} = 1$ bit. The redundant information about $Z$ that is contained in both $A$ and $B$ is effectively contained in $Z_2$ and is given by $\mathrm{Red}(Z: (A, B))=\mathrm{I}(Z;Z_2)=1$ bit. Lastly, the synergistic information about $Z$ that is not contained in either $A$ or $B$ alone, but is contained in both of them together is effectively contained in the tuple $(Z_3\oplus N,N)$, and is given by $\syn{Z}{(A,B)} = \mut{Z}{(Z_3\oplus N,N)}=1 $ bit. This accounts for the $3$ bits in $\mut{Z}{ (A,B)}$. Here, $B$ does not have any unique information about $Z$ that is not contained in $A$, i.e., $\uni{Z}{B| A} =0.$

Irrespective of the formal definition of the individual PID terms, the following identities also hold:
\begin{align}
&\mut{Z}{A}=\uni{Z}{A| B} + \red{Z}{(A, B)}. \label{eq:pid2}\\
&\mut{Z}{A \given B}=\uni{Z}{A| B} + \syn{Z}{(A, B)}.  \label{eq:pid3} 
\end{align}

Note that, $\uni{Z}{A| B}$ can be viewed as the information-theoretic sub-volume of the intersection between $\mut{Z}{A}$ and $\mut{Z}{A \given B}$. Similarly, $\red{Z}{(A, B)}$ is the sub-volume between $\mut{Z}{A}$ and $\mut{Z}{B}$.

These equations also demonstrate that $\uni{Z}{A| B}$ and  $\red{Z}{(A, B)}$ are the information contents that exhibit themselves in $\mut{Z}{A}$ which is the statistically visible information content about $Z$ present in $A$. Because both of these PID terms are non-negative, if any one of them is non-zero, we will have $\mut{Z}{A}>0$. Similarly, $\uni{Z}{B| A}$ and  $\red{Z}{(A, B)}$ also exhibit themselves in $\mut{Z}{B}$. On the other hand, $\syn{Z}{(A, B)}$ is the information content that does not exhibit itself in $\mut{Z}{A}$ or $\mut{Z}{B}$ individually, i.e., these terms can still be $0$ even if $\syn{Z}{(A, B)}>0$. But, $\syn{Z}{(A, B)}$ exhibits itself in $\mut{Z}{(A, B)}$. Also note that,
\begin{align}
\mut{Z}{(A, B)} &= \underbrace{\uni{Z}{A| B}  + \red{Z}{(A, B)}}_{\mut{Z}{A}} \nonumber \\
& + \underbrace{\uni{Z}{B| A} + \syn{Z}{(A, B)}}_{\mut{Z}{B\given A}} \\
& =\underbrace{\uni{Z}{B| A}  + \red{Z}{(A, B)}}_{\mut{Z}{B}} \nonumber \\
&+ \underbrace{\uni{Z}{A| B} + \syn{Z}{(A, B)}}_{\mut{Z}{A\given B}}.
\end{align}

Given three independent equations \eqref{eq:pid1}, \eqref{eq:pid2} and \eqref{eq:pid3} in four unknowns (the four PID terms), defining any one of the terms (e.g., $\uni{Z}{A| B}$) is sufficient to obtain the other three. For completeness, we include the definition of unique information from \cite{bertschinger2014quantifying} (that also allows for estimation via convex optimization~\cite{banerjee2018computing}) with the specific properties used in the proofs in the Appendix. To follow the paper, only an intuitive understanding (from Fig.~\ref{fig:pid}) may be sufficient.
\begin{defn}[Unique Information~\cite{bertschinger2014quantifying}] Let $\Delta$ be the set of all joint distributions on $(Z,A,B)$ and $\Delta_p$ be the set of joint distributions with the same marginals on $(Z,A)$ and $(Z,B)$ as their true distribution, \textit{i.e.}, $$\Delta_p=\{ Q \in \Delta : q(z,a){=}\Pr(Z{=}z,A{=}a) \text{ and } q(z,b){=}\Pr(Z{=}z,B{=}b) \}.$$ 
\begin{align}
\uni{Z}{A| B}=\min_{Q \in \Delta_p} \mathrm{I}_{Q}(Z;A\given B),
\end{align}
where $\mathrm{I}_{Q}(Z;A\given B)$ is the conditional mutual information when $(Z,A,B)$ have joint distribution $Q$. \label{defn:uni} 
\end{defn} 
\noindent The key intuition behind this definition is that the unique information should only depend on the marginal distribution of the pairs $(Z, A)$ and $(Z, B)$. This is motivated from an \textbf{operational} perspective that if $A$ has unique information about $Z$ (with respect to $B$), then there must be a situation where one can predict $Z$ better using $A$ than $B$ (more details in \cite[Section 2]{bertschinger2014quantifying}). Therefore, all the joint distributions in the set $\Delta_p$ with the same marginals essentially have the same unique information, and the distribution $Q^*$ that minimizes $\mathrm{I}_{Q}(Z;A\given B)$ is the joint distribution that has no synergistic information leading to $\mathrm{I}_{Q^*}(Z;A\given B)= \uni{Z}{A| B}$. Definition~\ref{defn:uni} also defines $\red{Z}{(A, B)} $  and $\syn{Z} {(A, B)}$ using \eqref{eq:pid2} and \eqref{eq:pid3}.

\begin{defn}[Redundant Information~\cite{bertschinger2014quantifying}] \label{defn:red} The redundant information about $Z$ contained in both $A$ and $B$ is given by: 
\begin{align}
\red{Z}{(A,B)}=\mut{Z}{A}-\uni{Z}{A\given B}.
\end{align} 
\end{defn}

\subsection{Background on Shapley Values}
\label{background:shapley}
Here, we provide a background on Shapley values~\cite{shapley1953}, a concept from cooperative game theory, that we use to arrive at our techniques of explainability. The Shapley value is a way of distributing a total reward generated by the coalition of all players. The Shapley values are a unique set of values satisfying certain assumptions~\cite{shapley1953}.

Consider a set $\mathcal{N}$ (of $|\mathcal{N}|=n$ total players) and a function $v(S)$ that maps a subset of players $S\subseteq \mathcal{N}$ to a real number. Here, $ v(\emptyset )=0$, where $ \emptyset $  denotes the empty set, and $v(\mathcal{N})$ is the total reward generated by the coalition of all players. In general, $v(S)$ is called the worth of coalition $S$, that describes the total expected reward that the members of $S$ can obtain by cooperation. The function $v$ is called a characteristic function.

According to the Shapley values, the  reward-amount $\varphi_{i}(v)$ that player $i$ is allocated in a coalition game specified by $(v,\mathcal{N})$ is:
$\varphi_{i}(v)=\sum_{S\subseteq \mathcal{N}\setminus \{i\}}\frac {|S|!\;(n-|S|-1)!}{n!} \big(v(S\cup \{i\})-v(S)\big).$
Note that the sum extends over all subsets $S$ of $\mathcal{N}$ not containing player $i$. 

The Shapley values satisfy many desirable properties, of which we include the most relevant one here.
\begin{propty}
The sum of the Shapley values of all players equals the total reward, i.e.,
$\sum _{i\in \mathcal{N} }\varphi_{i}(v)=v(\mathcal{N}).$
\label{propty:shapley}
\end{propty}

\section{MAIN RESULTS}
\label{sec:main_results}

This section is organized as follows: we begin (Section~\ref{subsec:candidate_measures}) with a canonical example that illustrates the difference between \emph{interventional} and \emph{distributional} approaches to explaining contributions of individual features to the observed disparity. Next, in Section~\ref{subsec:property}, we discuss some properties of our proposed \emph{distributional} measure of contribution to disparity, focusing on the case when one does not have access to the underlying decision-making mechanism, e.g., with human-in-the-loop. In Section~\ref{subsec:more_examples}, we discuss some more canonical examples to further illustrate the differences between our two proposed approaches.

\subsection{Interventional and Distributional Approaches to Explaining Contribution}
\label{subsec:candidate_measures}

We consider the following canonical example that represents scenarios where two or more features are highly correlated with each other.

\begin{cexample}[Highly Dependent Input Features]
Consider two highly dependent features\footnote{Often, such features are regarded as co-linear features.} being used for deciding student admissions. These two features are the scores in two courses: $X_1=Z+U$ and $X_2=Z+U$. Here $Z$ is the protected attribute, distributed as Bern(\nicefrac{1}{2}), and $U$ is another independent random variable with distribution Bern(\nicefrac{1}{2}) that could potentially represent inner ability (inspired from existing works in causal fairness, e.g.,~\cite{dutta2021fairness}). Now, suppose that the admission decision happens to be based on a score given by: $\hat{Y}=X_1$.
\end{cexample}

Let us now introduce our first approach to quantify the contribution of the features $X_1$ and $X_2$ when we have access to the model. Inspired from literature in explainability, we can vary an individual feature while keeping the other features unchanged, and assign the resulting change in disparity (mutual information between $Z$ and the output) to this feature. So, in essence, the change due to $X_1$ in the observed disparity $\mut{Z}{\hat{Y}}$ can be quantified as:
$$ \mut{Z}{\hat{Y}(X)}- \mut{Z}{\hat{Y}(X\backslash X_1)}  ,$$ where $\hat{Y}(X_S)$ denotes the output of the model/decision-making system that is based only on the features in the set $X_S\subseteq X$. In existing works on explainability~\cite{lundberg2017unified}, $\hat{Y}(X_S)$ is typically computed by setting the input features in the set $X\backslash X_S$ as constants for all data-points (e.g., equal to their respective means), and then evaluating the output of the model. Notice that, $\hat{Y}(X)$ is thus essentially $\hat{Y}$ which is the actual model output, and $\hat{Y}(\emptyset)$ is a constant. Thus, $\mut{Z}{\hat{Y}(X)}=\mut{Z}{\hat{Y}}$ and $\mut{Z}{\hat{Y}(\emptyset)}=0$.

A Shapley-value-based approach (building on this idea) would take an average of the change due to $X_i$ across all possible subsets $X_S\subseteq X\backslash X_i$, leading to the following candidate measure:
\begin{candmeas}[Interventional Contribution]
\begin{align}
\mathrm{Contri}(X_i)=&\sum_{X_S\subseteq X\setminus X_i}\frac {|X_S|!\;(n-|X_S|-1)!}{n!}\nonumber \\& (\mut{Z}{\hat{Y}(X_S\cup X_i)}-\mut{Z}{\hat{Y}(X_S)}).
\end{align}
\label{candmeas:interventional}
\end{candmeas}
Using this candidate measure of contribution for this canonical example, we get:
\begin{align}
&\mathrm{Contri}(X_1)=1/2 \ \text{bits.}\\
&\mathrm{Contri}(X_2)=0 \ \text{bits.} \\ & \mathrm{Contri}(X_1)+\mathrm{Contri}(X_2)= \mut{Z}{\hat{Y}}=1/2 \text{ bits}.
\end{align}
This perspective is in alignment with the actual mechanism of how the decision is being made, which is based only on $X_1$. However, based on this explanation of contributions to disparity, the admissions committee could choose to drop the first feature, and retrain their model using the remaining features (here, $X_2$) in the hope that it would remove disparity. This would turn out to be misleading since, now, disparity can still arise from $X_2$. Thus, from a distributional perspective, we might want to capture this statistical redundancy between these two features $X_1$ and $X_2$, even if the second feature is not actually being used in the mechanism.


Furthermore, consider a scenario where an external observer does not have access to the actual mechanism of the model or how the decision is being made (also see Section~\ref{system_model}). They do not know if the decision is a deterministic function of the input $X$, or if the decision is made with human-in-the-loop and their subjective evaluation as well (since the human has access to both of the features along with additional non-quantified aspects). The observer only gets to observe the inputs to the model, and the final output decisions for a dataset. In such a situation, it is not possible to intervene on the inputs to the decision-making system and compare the outputs for different values of the input. In fact, even if one retrains the original model with $X_1$ and $X_2$ as inputs, it may not converge to the same weights because of the \emph{distributional redundancy} between the two features. Then, the question that for the observer is: \emph{what is the ``potential'' contribution of each individual feature to the observed disparity $\mut{Z}{\hat{Y}}$ when the exact decision-making mechanism is not accessible?}

For answering this question, in this work, we first examine an information-theoretic quantity called redundant information (recall Definition~\ref{defn:red} in Section~\ref{background:pid}). Redundant information captures the redundant statistical dependency about $Z$ present in the decision $\hat{Y}$ and a set of features $X_S$, and is written as $\red{Z}{(\hat{Y},X_S)}$.

A Shapley-value-based approach would take an average of the change in the redundant information due to $X_i$ across all possible subsets $X_S\subseteq X\backslash X_i$, leading to the following: 
\begin{candmeas}[Potential Contribution]
\label{candmeas:distributional}
\begin{align}
&\mathrm{PotentContri}(X_i)=\sum_{X_S\subseteq X\setminus X_i}\frac {|X_S|!\;(n-|X_S|-1)!}{n!}\nonumber \\ &(\red{Z}{(\hat{Y},X_S\cup X_i)}-\red{Z}{(\hat{Y},X_S)}).
\end{align}
\end{candmeas}
Now, let us see what $\mathrm{PotentContri}(X_i)$ signifies for this canonical example.  Notice that,
\begin{align}
&\red{Z}{(\hat{Y},X_1)}\overset{(a)}{=}\mut{Z}{\hat{Y}}-\uni{Z}{\hat{Y} \given X_1}\nonumber \\
&\overset{(b)}{=}\mut{Z}{\hat{Y}}=1/2 \text{ bits}.
\end{align}
Here (a) holds from \eqref{eq:pid2}, and (b) holds because $\uni{Z}{\hat{Y} \given X_1}\leq \mut{Z}{\hat{Y}\given X_1}=0$ (non-negativity of PID and \eqref{eq:pid2}; see Lemma 1 in Appendix). Because $\red{Z}{(\hat{Y},X_1)}$ is a purely distributional quantity, 
\begin{equation}
\red{Z}{(\hat{Y},X_2)}=\red{Z}{(\hat{Y},X_1)}=1/2\text{ bits}.
\end{equation}
We also have $\red{Z}{(\hat{Y},X)}=1/2$ bits, and $\red{Z}{(\hat{Y},\emptyset)}=0$ bits. This leads to:
\begin{align}
&\mathrm{PotentContri}(X_1)=1/2(\red{Z}{(\hat{Y},X)} \nonumber \\
&-\red{Z}{(\hat{Y},X_2)})+ 1/2(\red{Z}{(\hat{Y},X_1)}\nonumber \\
&-\red{Z}{(\hat{Y},\emptyset)}) =1/4 \ \text{bits.} \\ &\mathrm{PotentContri}(X_2)=1/2(\red{Z}{(\hat{Y},X)}\nonumber \\
&-\red{Z}{(\hat{Y},X_1)})+1/2(\red{Z}{(\hat{Y},X_2)} \nonumber \\
&-\red{Z}{(\hat{Y},\emptyset)}) =1/4 \ \text{bits.} \\ & \mathrm{PotentContri}(X_1)+\mathrm{PotentContri}(X_2) \nonumber \\
&= \mut{Z}{\hat{Y}}=1/2 \text{ bits}.
\end{align}

\subsection{Properties of our Distributional Measure of Contribution}
\label{subsec:property}
Our distributional approach to quantifying contributions can identify if any additional disparity has been introduced in the decision-making process from non-quantified aspects, e.g., due to human-in-the-loop who almost always have access to the protected attributes.

\begin{thm}[Unexplained Disparity]
In general, we have $\mathrm{PotentContri}(X_i) \leq \mut{Z}{\hat{Y}}.$ An equality holds if and only if $\uni{Z}{\hat{Y}\given X}=0$.
\label{thm:unexplained}
\end{thm}

\begin{proof}
In general, we have, $$\sum_{X_i\in X}\mathrm{PotentContri}(X_i)= \red{Z}{(\hat{Y},X)}$$ (property of Shapley values; see Property~\ref{propty:shapley} in Section~\ref{background:shapley}). Now, observe that,
\begin{align*}
\red{Z}{(\hat{Y},X)}=\mut{Z}{\hat{Y}}-\uni{Z}{\hat{Y}\given X} \overset{(a)}{\leq} \mut{Z}{\hat{Y}},
\end{align*}
where (a) holds from the non-negativity of PID terms. An equality holds if and only if $\uni{Z}{\hat{Y}\given X}=0$.
\end{proof}

\begin{rem}[Significance of Theorem~\ref{thm:unexplained}]
Since the observer does not have access to the model, they do not know if the decision $\hat{Y}$ is an entirely deterministic function of the input $X$ (or not, e.g., because of human-in-the-loop). When $\hat{Y}$ is entirely determined by input $X$, the Markov chain $Z-X-\hat{Y}$ holds, implying $\uni{Z}{\hat{Y}\given X}=0$ (Lemma 1 in Appendix). The sum of the individual contributions of all the features also add up to the observed disparity $\mut{Z}{\hat{Y}}$. However, when $\hat{Y}$ is not an entirely deterministic function of the input $X$,  $\uni{Z}{\hat{Y}\given X}$ may be non-zero. This signifies that there is some unique information about $Z$ in the output $\hat{Y}$ that cannot be attributed to the input features $X$, implying additional unexplained disparity introduced in the decision-making process.
\end{rem}

\begin{thm}[Bounds on PotentContri]
For any feature $X_i$, we have $0 \leq \mathrm{PotentContri}(X_i)\leq \mut{Z}{\hat{Y}}$. 
\end{thm}

\begin{proof}
For the lower-bound, it suffices to show that $$\red{Z}{(\hat{Y},X_S\cup X_i)} \geq \red{Z}{(\hat{Y},X_S)}$$ for any set $X_S \subseteq X \backslash X_i$. Note that,
\begin{align}
&\red{Z}{(\hat{Y},X_S\cup X_i)} \nonumber \\
&= \mut{Z}{\hat{Y}}- \uni{Z}{\hat{Y}\given X_S\cup X_i }\\
& \overset{(a)}{\geq} \mut{Z}{\hat{Y}}- \uni{Z}{\hat{Y}\given X_S } \\
& = \red{Z}{(\hat{Y},X_S)}.
\end{align}
Here (a) holds from a property of unique information (see Lemma 2 in Appendix).

For the upper bound, notice that 
\begin{align}
& \mathrm{PotentContri}(X_i) \overset{(a)}{\leq} \sum_{X_i\in X}\mathrm{PotentContri}(X_i) \nonumber \\
& \overset{(b)}{=} \red{Z}{(\hat{Y},X)} \overset{(c)}{\leq} \mut{Z}{\hat{Y}}.
\end{align}
Here (a) holds from non-negativity of $\mathrm{PotentContri}(X_i)$ (proved just above), (b) holds from Theorem~\ref{thm:unexplained}, and (c) holds since PID is a non-negative decomposition.
\end{proof}

Next, we examine more canonical examples to better understand the \emph{interventional} and \emph{distributional} approaches.

\subsection{More Canonical Examples}
\label{subsec:more_examples}

\begin{cexample}[Disparity Amplification]
{Consider two features: score in two courses, $X_1=Z+U$ and $X_2=U$. Here $Z$ is the protected attribute, distributed as Bern(\nicefrac{1}{2}), and $U\sim$ Bern(\nicefrac{1}{2}) is another independent random variable (e.g. representing inner ability). Now, suppose  the decision is: $\hat{Y}=X_1- X_2=Z$.}
\end{cexample}

This is a case of disparity amplification. Here, though $X_2$ does not individually have any information about $Z$ (indeed $\mut{Z}{X_2}=0$), it leads to an increase in the observed disparity in the final output $\hat{Y}$. We have $\mut{Z}{\hat{Y}}=1\text{ bit}$ which would not have been possible from using the feature $X_1$ alone, since $\mut{Z}{X_1}$ is much less. Thus, we would want $X_2$ to also be assigned a contribution to the overall disparity.

For this example, the interventional perspective to quantifying contributions would lead to the following explanations:
\begin{align}
&\mathrm{Contri}(X_1)=3/4 \ \text{bits.}\\
&\mathrm{Contri}(X_2)=1/4 \ \text{bits.} \\ & \mathrm{Contri}(X_1)+\mathrm{Contri}(X_2)= \mut{Z}{\hat{Y}}=1 \text{ bit}.
\end{align}
Interestingly, the distributional perspective to quantifying ``potential'' contributions would also lead to the same:
\begin{align}
&\mathrm{PotentContri}(X_1)=3/4 \ \text{bits.}\\
&\mathrm{PotentContri}(X_2)=1/4 \ \text{bits.} \\ & \mathrm{PotentContri}(X_1)+\mathrm{PotentContri}(X_2)= \mut{Z}{\hat{Y}} \nonumber \\
&=1 \text{ bit}.
\end{align}

Thus, both the approaches are able to quantify the contribution of the individual features to the observed disparity, as desired.

\begin{cexample}[Disparity Masking]
{Let $Z\sim$ Bern(\nicefrac{1}{2}) be the protected attribute, and $U\sim$ Bern(\nicefrac{1}{2}) be another independent random variable  (e.g. representing inner ability). Consider two features: $X_1=Z$ (the protected attribute) and $X_2=U$ (score in a course).  Now, suppose that the admission decision is: $\hat{Y}=X_1\oplus X_2=Z\oplus U$. }
\end{cexample}

This is a case of masked discrimination: high-scoring candidates of the non-protected group and low-scoring candidates of the protected group are admitted. From an observational perspective, there is no disparity in the overall decisions. This is also exhibited by the fact that $\mut{Z}{\hat{Y}}=0$. However, a careful examination using interventions reveal that the decisions are discriminatory towards high-scoring candidates of the protected group who are not admitted.

For this example, the interventional perspective leads to the following contributions:
\begin{align}
&\mathrm{Contri}(X_1)=1/2 \ \text{bits and }\mathrm{Contri}(X_2)=-1/2 \ \text{bits.} \nonumber \\ & \mathrm{Contri}(X_1)+\mathrm{Contri}(X_2)= \mut{Z}{\hat{Y}}=0 \text{ bits}.
\end{align}
On the other hand, the distributional perspective for  ``potential'' contributions leads to the following:
\begin{align}
&\mathrm{PotentContri}(X_1)=0 \ \text{bits.}\\
&\mathrm{PotentContri}(X_2)=0 \ \text{bits.} \\ & \mathrm{PotentContri}(X_1)+\mathrm{PotentContri}(X_2) 
= \mut{Z}{\hat{Y}}=0 \text{ bits}.
\end{align}

\begin{rem}[On Negative Contribution] We note that while $\mathrm{PotentContri}$ is always non-negative, $\mathrm{Contri}$ can take negative values. Here, a negative value of $\mathrm{Contri}$ demonstrates that a feature ($X_2$), in fact, reduces (effectively masks) the overall disparity that was introduced by another feature ($X_1$).
\end{rem}

\begin{figure}
\centering
\includegraphics[height=5cm]{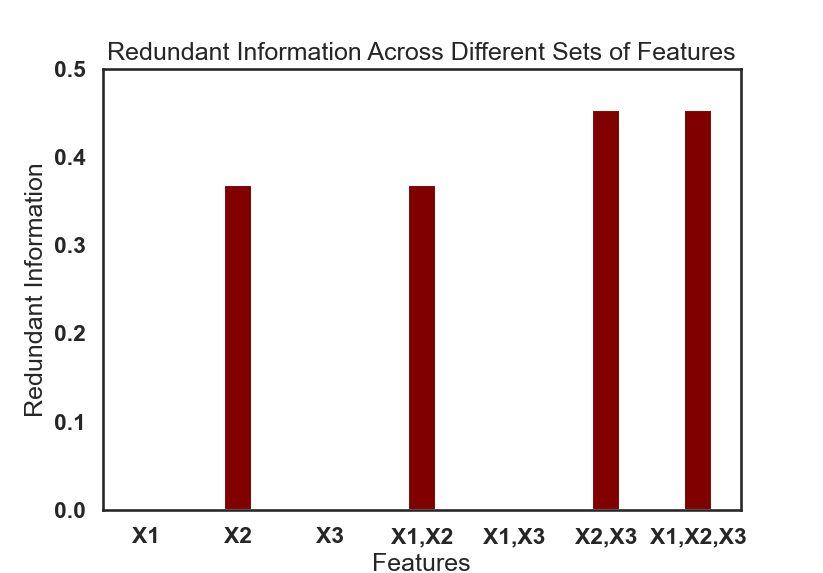}
\caption{\small{Understanding redundant information about gender $Z$ in $\hat{Y}$ and subset $S\subseteq (X_1,X_2,X_3)$ across different subsets of features, i.e., $\red{Z}{(\hat{Y},S)}$: Notice that, the features $(X_2,X_3)$ jointly have higher redundant information about $Z$ with $\hat{Y}$ than either of the features alone. \label{fig:redundant} }}
\end{figure}

\section{CASE STUDY}
\label{sec:case_study}
Here, we include a simple case study on an artificial dataset to demonstrate how to compute potential contributions of features on datasets.

\begin{example} We consider a scenario of college admissions with three features: discrete scores corresponding to the GPA, GRE, and Recommendation Letters. Let $Z \sim$ Bern(\nicefrac{1}{2}) denote gender, $X_1=U_1\sim$ Bern(\nicefrac{9}{10}) denote whether GPA is above a desirable threshold, $X_2=Z+U_2$ (with $U_2\sim$ Bern(\nicefrac{1}{2})) denote a thresholded GRE score, and $X_3=U_2$  denote a score from reference letters. Here, $U_1$, and $U_2$ can be thought of as latent random variables denoting inner abilities of a candidate that are independent of $Z$. Suppose that the admission committee uses an automated figure of merit $F=Z+U_1+2U_2$ whose exact mechanism is not known to us (could be something like $F=X_1+X_2+X_3$ or $F=X_1+2X_2-X_3$). We can only observe the final decisions based on $\hat{Y}= F+R = Z+U_1+2U_2+N$ where $N\sim$ Bern(\nicefrac{1}{10}) denotes a manual evaluation score (which might be subjective and not a deterministic function of the inputs $X_1,X_2,X_3$).
\end{example}

In  Fig.~\ref{fig:redundant} and Fig.~\ref{fig:contri}, we employ our technique to quantify the potential contribution of each individual feature to the overall disparity $\mut{Z}{\hat{Y}}$. We use the discrete information theory (\texttt{dit})~\cite{dit} package to compute all of the PID terms. The package solves a convex optimization problem for the computation of unique information using the empirical distribution of the data ($10$K samples). 

\begin{rem}[Early Truncation for Scalability] We note that there are computational challenges as we scale this technique to larger number of features. To address this, we employ an approximation that is based on the following critical observation: $\red{Z}{(\hat{Y},X_S)}$ is non-decreasing as the number of features in the set $X_S$ increases (Lemma 3 in Appendix), and is upper bounded by $\mut{Z}{\hat{Y}}$. Thus, if we arrive at a value of $\red{Z}{(\hat{Y},X_S)}$ such that $|\red{Z}{(\hat{Y},X_S)}-\mut{Z}{\hat{Y}}|\leq \epsilon$ for a small choice of $\epsilon$, then we can approximate $|\red{Z}{(\hat{Y},X'_S\cup X_i)}-\red{Z}{(\hat{Y},X'_S)}|$ with $\epsilon$ for all supersets $X'_S \supseteq X_S$. In practice, this can reduce the number of sets $X_S$ on which we need to compute $\red{Z}{(\hat{Y},X_S)}$. 
\end{rem}

\begin{figure}
\centering
\includegraphics[height=5cm]{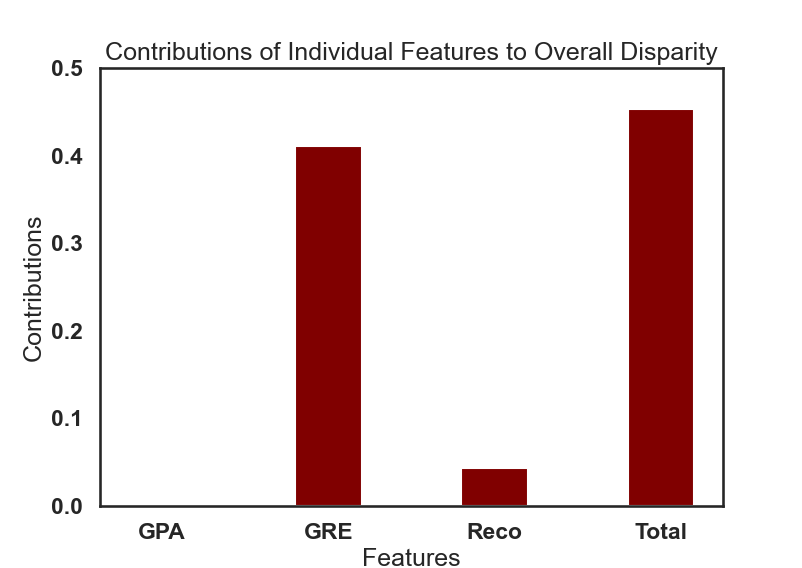}
\caption{\small{Quantifying Potential Contribution of Individual Features to the Overall Disparity $\mut{Z}{\hat{Y}}$ based on Redundant Information: Notice that, the feature Reco ($X_3$) is also assigned a potential contribution since $(X_2,X_3)$  jointly have a higher redundant information about $Z$ with $\hat{Y}$ than either of them alone.
\label{fig:contri}  }}
\end{figure}


\section{CONCLUSION AND BROADER IMPACT}

This work proposes two novel approaches to quantifying contribution of individual features to the overall disparity ($\mut{Z}{\hat{Y}}$) in a decision, and discusses when either is better suited. These techniques could be applicable for auditing decision-making mechanisms: when one is able to intervene or not intervene on the inputs, as well as in scenarios with human-in-the-loop. Future work will examine how these approaches can inform  intervention and repair of the decision-making mechanism to reduce disparity.


We note that the computational complexity of computing Shapley values definitely becomes a challenge for complex neural networks with very large number of features/neurons that are used in computer vision or natural language processing. In future work, we will also explore techniques of approximating Shapley values for faster computation, e.g., in \cite{ghorbani2020neuron}. However, we envision our technique to still be applicable for smaller models used in more consequential applications, e.g., college admissions, or credit decision that use simple models with a few features and human-in-the-loop~\cite{student_risk}.

Another future work could be improving the estimation of these quantities for continuous random variables (either by discretizing them~\cite{dutta2021fairness} or leveraging techniques in \cite{mukherjee2019ccmi,galhotra2020fair,venkatesh2021can,schamberg2021partial,pakman2021estimating} for the ease of computation).

\small{
\bibliographystyle{IEEEtran}
\bibliography{aaai22}
}

\section*{Disclaimer}
This paper was prepared for informational purposes and is not a product of the Research Department of J.P. Morgan. J.P. Morgan makes no representation and warranty whatsoever and disclaims all liability, for the completeness, accuracy or reliability of the information contained herein. This document is not intended as investment research or investment advice, or a recommendation, offer or solicitation for the purchase or sale of any security, financial instrument, financial product or service, or to be used in any way for evaluating the merits of participating in any transaction, and shall not constitute a solicitation under any jurisdiction or to any person, if such solicitation under such jurisdiction or to such person would be unlawful.

\section*{Appendix}

\begin{lem}[Zero Unique Information]
When the Markov chain $Z-X-\hat{Y}$ holds, we have $\uni{Z}{\hat{Y}\given X}=0$. 
\label{lem:markov}
\end{lem}

From the definition of PID, we have, $\mut{Z}{\hat{Y} \given X}=\uni{Z}{\hat{Y}| X} + \syn{Z}{(\hat{Y}, X)}$ where both the terms are non-negative. Thus,
$\uni{Z}{\hat{Y}| X} \leq \mut{Z}{\hat{Y} \given X}$. Notice that, when the Markov chain $Z-X-\hat{Y}$ holds, we have $\mut{Z}{\hat{Y}\given X}=0$, thus implying $\uni{Z}{\hat{Y}| X} =0$ as well.

\begin{lem}[Monotonicity of Unique Information]
For all $(Z,B,X_c,X_c')$, we have:
$$\uni{Z}{B| X_c \cup X_c'} \leq \uni{Z}{B| X_c}   .$$
\label{lem:monotonicity_uni}
\end{lem}

This result is derived in  \cite[Lemma 32]{banerjee2018unique}. 

\begin{lem}[Monotonicity of Redundant Information]
For all $(Z,B,X_c,X_c')$, we have:
$$\red{Z}{(B, X_c \cup X_c')} \geq \red{Z}{(B, X_c)}   .$$
\label{lem:monotonicity_red}
\end{lem}

The proof holds from Lemma 2, since $\uni{Z}{B| X_c \cup X_c'} + \red{Z}{(B, X_c \cup X_c')}= \mut{Z}{B} = \uni{Z}{B| X_c} + \red{Z}{(B, X_c)}.$

\end{document}